%% file: main.tex
\RequirePackage[T1]{fontenc}
 \documentclass[conference]{IEEEtran}
\IEEEoverridecommandlockouts
\usepackage{amsmath,amssymb,amsthm}
\usepackage{graphicx,booktabs,cite,microtype}
\usepackage[hidelinks]{hyperref}
\graphicspath{{figs/}}
\newcommand{\Nr}{N_r}
\newcommand{\dN}{d_{90}}

\newcommand{\R}{\mathbb{R}}
\DeclareMathOperator{\logit}{logit}
\theoremstyle{definition}
\newtheorem{definition}{Definition}
\theoremstyle{plain}
\newtheorem{proposition}{Proposition}
\theoremstyle{remark}
\newtheorem{remark}{Remark}
\hypersetup{pdftitle={Emergency Obstacle Avoidance Maneuvers in Differential-Drive Mobile Robots},pdfauthor={Kautik Mandve and Dr. Nilay Kant},pdfsubject={Experimental comparison of emergency maneuvers, footprint clearance, and odometric error}}
\title{\LARGE \bf Emergency Obstacle Avoidance Maneuvers in \\Differential-Drive Mobile Robots}
\author{Kautik Mandve and Dr. Nilay Kant
\thanks{Kautik Mandve is a high-school student. Dr. Nilay Kant is with the Department of Mechanical and Aerospace Engineering, Missouri University of Science and Technology, Rolla, MO 65409, USA.}
\thanks{Corresponding author: Dr. Nilay Kant (\texttt{nilaykant@mst.edu}).}}
\input{contrast_rows.tex}
\input{d90_rows.tex}

\begin{document}
\raggedbottom
\maketitle

\begin{abstract}
Emergency obstacle avoidance requires a mobile robot to brake or change its heading within the available distance. This paper investigates the dependence of maneuver performance and odometric error on approach speed for a differential-drive robot. A footprint-clearance analysis and a normalized wheel-motion index provide a kinematic description of five braking and turning maneuvers. The principal experiment comprises 540 block-randomized trials on tile, of which 539 are retained, at commanded approach speeds of 45, 55, and 65 cm/s. An overhead camera provides an independent pose reference. When encoder, gyro, and camera heading changes are evaluated over complete motion records, the mean encoder discrepancy increases by 4.7--5.1 degrees for the two reverse-spin maneuvers between the lowest and highest speeds; the arc and brake-assisted pivot change by less than 0.4 degrees. Larger encoder discrepancy is associated with lower avoidance success after adjustment for distance, speed, maneuver, day, and turn direction. Gyro discrepancies remain approximately 3 degrees for the reverse spins. Estimated reaction distances for 90\% success and their uncertainty quantify the maneuver trade-offs. The results support speed-specific empirical characterization of emergency maneuvers and distinguish encoder error from recovery motion and measurement-window mismatch.
\end{abstract}

\section{Introduction}
Emergency obstacle avoidance requires a choice between reducing forward motion and redirecting it. The decision depends on the available reaction distance, the robot footprint, and the motion that the actuators can realize. Dynamic-window methods restrict candidate velocities using acceleration and stopping constraints \cite{fox1997}; velocity obstacles describe velocities leading to future collisions \cite{fiorini1998}. Vehicle threat-assessment methods explicitly compare braking, steering, and acceleration maneuvers \cite{brannstrom2010}. Reachability and control barrier functions provide complementary ways to characterize admissible motion under specified models and constraints \cite{bansal2017,ames2019}. Experimental characterization of execution error can inform these methods when nominal wheel motion differs from body motion.

Wheel odometry is particularly vulnerable to departures from the rolling assumption. Borenstein and Feng distinguish systematic calibration errors from nonsystematic errors caused by slip and surface irregularities \cite{borenstein1996umb}; combining gyros and wheel encoders provides additional heading information \cite{borenstein1996gyro}. Chong and Kleeman develop an odometry error model for mobile-robot motion estimation \cite{chong1997}. These contributions establish the importance of separating calibration and measurement uncertainty from departures of wheel motion from body motion.

Related work on tracked and wheeled skid-steer robots explicitly accounts for nonideal turning kinematics. Mart\'inez \emph{et al.} identify approximate tracked-vehicle models through effective instantaneous centers of rotation \cite{martinez2005}, and Mandow \emph{et al.} extend this experimental approach to wheeled skid-steer vehicles \cite{mandow2007}. Dogru and Marques incorporate terrain and center-of-mass effects and evaluate different terrains and asymmetric loads \cite{dogru2021}. Yi \emph{et al.} combine skid-steer kinematics with low-cost inertial measurements for motion estimation \cite{yi2009}. These platforms have different contact constraints from the two-drive-wheel robot considered here; their identified parameters cannot be transferred directly. They nevertheless motivate measuring actual motion when evaluating nominal wheel profiles.

Slip estimation has also been studied using dynamic models \cite{ward2008}, motor current \cite{ojeda2006}, online odometry calibration \cite{degiorgi2024}, and learned relationships between IMU and encoder signals in outdoor terrain \cite{nourizadeh2023}. Tire-dynamics treatments describe the force constraints underlying combined braking and steering \cite{pacejka2012,rajamani2012}, while Kritayakirana and Gerdes study autonomous control at the limits of handling \cite{kritayakirana2012}. This literature supplies modeling and sensing approaches for nonideal motion. The present experiment instead characterizes the clearance and heading discrepancies of fixed emergency-maneuver profiles across approach speeds. It does not identify a tire-force model or evaluate a new online slip estimator.

High wheel-speed contrast can produce a rapid heading change, but does not by itself determine the distance needed to avoid an obstacle. The rotating footprint may advance toward the obstacle even if the robot reference point remains fixed. Incoming motion, transient braking, and subsequent recovery add further displacement. Moreover, a turn terminated at a prescribed encoder count need not produce the intended final body heading. These effects motivate joint measurement of clearance, net heading discrepancy, and approach speed.

This paper presents a geometric analysis of footprint clearance and a controlled comparison of five maneuvers. A normalized wheel-motion index describes the nominal profiles without interpreting them as measured tire forces. The principal dataset contains 539 usable trials from a balanced maneuver--speed experiment. Encoder and gyro heading changes are compared with the camera over complete motion records, and reaction-distance requirements are estimated with uncertainty. Two earlier campaigns, comprising 300 and 210 trials, respectively, motivated the controlled design; they are not pooled with the principal dataset. The resulting characterization relates maneuver selection to reaction distance and measured sensing error.

The remainder of the paper is organized as follows. Section~\ref{sec:problem} formulates the clearance and heading-measurement problems. Section~\ref{sec:analysis} develops the kinematic analysis. Section~\ref{sec:experiment} describes the apparatus and statistical methods. Section~\ref{sec:results} presents the results. Sections~\ref{sec:discussion} and~\ref{sec:conclusion} discuss their implications and conclude the paper.

\section{Problem Formulation}\label{sec:problem}
\begin{figure*}[t]
 \centering
 \includegraphics[width=\linewidth]{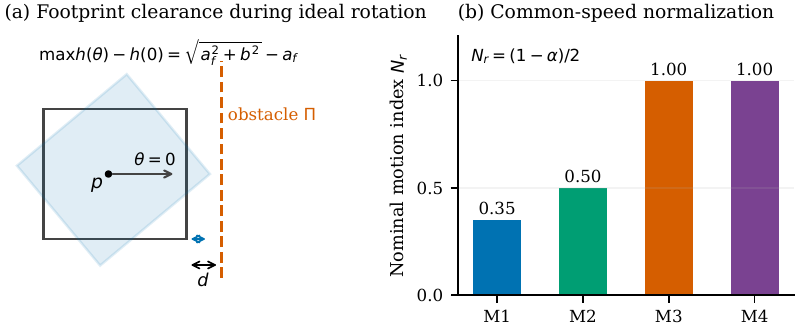}
 \caption{Kinematic description. (a) A rectangular footprint at the initial heading and at the heading of maximum forward support. The reference point is fixed in this idealized illustration; $d$ denotes initial front clearance. (b) Nominal wheel-motion indices for a common outer-wheel speed. The normalization describes wheel profiles and does not represent measured friction utilization.}
 \label{fig:geometry}
\end{figure*}

\subsection{Robot motion and obstacle geometry}
Let $p=[x,y]^\top\in\R^2$ and $\theta$ denote the planar robot reference position and heading. For track width $L$ and wheel peripheral velocities $u_L,u_R$, ideal rolling gives
\begin{equation}
 \dot p=v[\cos\theta,\sin\theta]^\top,\quad
 v=\frac{u_R+u_L}{2},\quad \dot\theta=\frac{u_R-u_L}{L}.
 \label{eq:kin}
\end{equation}
These relations are nominal kinematics; encoder-derived peripheral velocities are not measurements of ground-relative wheel velocity during slip.

The obstacle is represented by a virtual line in the floor plane,
$\Pi=\{z:n^\top z=c\}$, where $n$ is a unit normal pointing toward the obstacle. The admissible side satisfies $n^\top z\leq c$. Let $\mathcal B$ denote a rectangular footprint expressed relative to the tracked reference point, with front offset $a_f$, rear offset $a_r$, and half-width $b$. With $R(\theta)$ the planar rotation matrix, define clearance by
\begin{equation}
 C(t)=c-n^\top p(t)-\max_{r\in\mathcal B}n^\top R(\theta(t))r.
 \label{eq:clearance}
\end{equation}
The maximum occurs at a footprint corner. The experimental minimum $C_{\min}$ is the minimum of (\ref{eq:clearance}) over valid camera frames. A trial succeeds if $C_{\min}\geq0$. This is a sampled footprint-crossing criterion for a virtual obstacle; physical impacts and obstacle detection are not tested.

The initial heading is normal to the plane. The robot approaches at a commanded speed $v_c$ and initiates a selected maneuver after a prescribed encoder-measured travel distance. We denote the commanded front-bumper reaction distance by $d$. It is the prescribed reference-point distance minus $a_f$, and need not equal the camera-measured clearance at the instant of triggering. All fitted reaction-distance results below use this commanded coordinate.

\subsection{Heading measurements}
Let $\Theta_c$ be the unwrapped camera heading change from the initial stationary pose to the final stationary pose. The corresponding whole-record encoder and gyro changes are
\begin{equation}
 \Theta_e=\frac{1}{L}\int(u_R-u_L)\,dt,\qquad
 \Theta_g=\int\omega_g\,dt,
 \label{eq:heading}
\end{equation}
where $\omega_g$ is the recorded bias-corrected gyro yaw rate. Each integral covers the recorded motion, including approach, maneuver, and recovery. Initial and final stationary intervals permit comparison of the net changes without aligning individual samples between devices.

\begin{definition}\label{def:heading}
For the single-direction turns considered here, the absolute net-heading discrepancies are
\begin{equation}
 E_e=\bigl||\Theta_e|-|\Theta_c|\bigr|,\qquad
 E_g=\bigl||\Theta_g|-|\Theta_c|\bigr|.
 \label{eq:error}
\end{equation}
Magnitudes accommodate the opposite camera and onboard yaw-sign conventions. Errors are reported in degrees.
\end{definition}
The signed quantity $|\Theta_e|-|\Theta_c|$ distinguishes encoder overestimation from underestimation; $E_e$ measures their magnitude. Neither an absolute endpoint difference nor a net integral measures the accumulated absolute slip over the trajectory. Encoder scale, track-width error, and unobserved body motion can all contribute to $E_e$.

\subsection{Experimental objective}
The objective is to determine, for the tested robot and surface, how approach speed changes (i) net encoder and gyro discrepancies, (ii) the association between encoder discrepancy and avoidance success, and (iii) the distribution of minimum clearance and required reaction distance. The maneuver parameters are fixed within the principal campaign. The analyses are exploratory.

\section{Kinematic Analysis of Emergency Maneuvers}\label{sec:analysis}
\subsection{Footprint cost of an ideal rotation}
An ideal rotation about a fixed reference point does not necessarily preserve bumper clearance. The following calculation isolates the footprint contribution from translational carry.

\begin{proposition}\label{prop:geometry}
Suppose the reference point remains fixed while the robot rotates from a heading normal to $\Pi$ through $\pi/2$. For $a_f\geq0$ and $b\geq0$, the initial front clearance required solely by the rotating rectangular footprint is
\begin{equation}
 d_{\rm rot}=\sqrt{a_f^2+b^2}-a_f.
 \label{eq:sweep}
\end{equation}
\end{proposition}
\begin{proof}
For $0\leq\theta\leq\pi/2$, the forward support of the footprint is $h(\theta)=a_f\cos\theta+b\sin\theta$. Its maximum is $\sqrt{a_f^2+b^2}$, attained at $\theta=\tan^{-1}(b/a_f)$ when $a_f>0$. Subtracting $h(0)=a_f$ gives (\ref{eq:sweep}); the limiting cases follow directly.
\end{proof}
For the experimental footprint, $a_f=9.0$ cm and $b=7.5$ cm, so $d_{\rm rot}=2.72$ cm. This is a geometric reference for the stated ideal motion, not a bound for every actual maneuver: translation away from or toward the plane changes the required clearance. For a left turn the same result follows by symmetry.

For straight braking with constant deceleration magnitude $a_b$ after latency $\tau$, the nominal stopping distance is
\begin{equation}
 d_b=v_0\tau+\frac{v_0^2}{2a_b}.
 \label{eq:brake}
\end{equation}
Neither $a_b$ nor $\tau$ is identified from the success curves in this study. Consequently, (\ref{eq:brake}) explains the expected speed dependence but is not used to infer a friction coefficient or motor-torque limit.

\subsection{Nominal wheel-motion index}
For a steady ideal wheel profile, the path-normal acceleration of wheel-center point $i$ is $|\omega u_i|$, with $\omega=(u_R-u_L)/L$. Consider a common nominal outer-wheel speed $U>0$ and inner-wheel speed $\alpha U$, where $-1\leq\alpha\leq1$.
\begin{definition}\label{def:index}
The normalized wheel-motion index is
\begin{equation}
 \Nr=\frac{\max_i|\omega u_i|}{2U^2/L}
     =\frac{1-\alpha}{2}.
 \label{eq:index}
\end{equation}
The denominator is the wheel-center normal-acceleration scale of the ideal counter-rotating profile $u_R=-u_L=U$.
\end{definition}
The nominal ratios for the arc, pivot, and two reverse spins give $\Nr=0.35$, $0.50$, $1.00$, and $1.00$, respectively. The index describes the nominal profile; controller transients and exit motion require empirical evaluation. M0 commands no rotation and is treated separately.

\begin{remark}\label{rem:force}
Equation~(\ref{eq:index}) is a kinematic normalization, not a friction-utilization ratio. A wheel-center acceleration cannot be multiplied by an assumed share of robot mass to identify the contact force without a rigid-body/contact model. Although tire-force constraints motivate examining aggressive maneuvers \cite{pacejka2012,rajamani2012}, no coefficient of friction, normal-load distribution, or saturation boundary is measured here. In particular, $\Nr=1$ denotes the chosen reference profile and does not establish tire saturation.
\end{remark}

\begin{figure}[t]
 \centering
 \includegraphics[width=\linewidth]{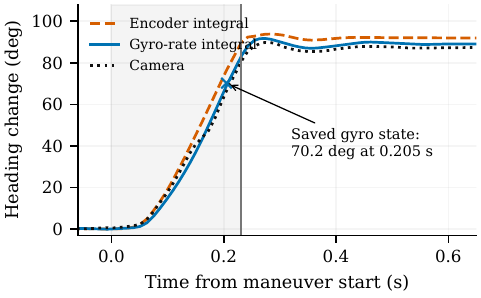}
 \caption{Illustrative M3 right turn at a 65 cm/s command. Shading denotes the commanded turn. The stored gyro summary of 70.2 degrees matches a sample 26 ms before turn completion; rotation also continues during recovery. Whole-record gyro integration follows the final camera heading more closely. Camera alignment uses cross-correlation ($r=0.992$); the population endpoint metrics do not require this alignment.}
 \label{fig:window}
\end{figure}

\subsection{Encoder termination and final heading}
Under (\ref{eq:kin}), an encoder-differential threshold corresponds to a prescribed nominal rotation. Under slip, however, the true angular velocity can differ from $(u_R-u_L)/L$. With a common yaw-sign convention and motion interval, the signed heading difference is
\begin{equation}
 \Theta_e-\Theta_c=\int\left[\frac{u_R-u_L}{L}-\dot\theta\right]dt.
 \label{eq:gap}
\end{equation}
The wheel signals alone cannot determine the integrand without additional information about body motion or slip. A gyro provides an independent yaw-rate measurement, but it must be evaluated over the same motion interval.

In particular, a firmware summary saved when the turn threshold is reached can contain an earlier gyro state and excludes subsequent recovery. Comparing it with final camera heading therefore confounds sampling time, recovery motion, and sensor error. Figure~\ref{fig:window} illustrates this distinction. The results use (\ref{eq:heading})--(\ref{eq:error}), with the encoder geometry separated from the firmware's turn-count calibration.

\section{Experimental Apparatus and Methods}\label{sec:experiment}
\subsection{Robot and maneuver implementation}
Experiments use a Pololu Romi 32U4 with 7.0 cm wheel diameter, 14.7 cm track width, and 1440 encoder counts per wheel revolution. A 16 MHz ATmega32U4 executes a 16 ms wheel-speed control loop. An LSM6DS33 gyro records yaw motion; neither the gyro nor the camera closes the heading loop. The measured top-speed reference is 77.5 cm/s. Approach speed is commanded in absolute cm/s, independently of the maneuver's nominal entry fraction.

All maneuvers use one parameterized execution engine. M0 brakes without a commanded turn. M1 drives the inner wheel forward to produce an arc. M2 brakes the inner wheel while driving the outer wheel. M3 commands opposite wheel motion and exits with zero target speed. M4 uses the same opposite-wheel profile with an earlier deceleration ramp and a nonzero exit target. Table~\ref{tab:maneuvers} gives the parameters. M3 and M4 thus differ in both deceleration shaping and exit speed; their comparison does not isolate a single parameter.

Turning wheel-speed targets can reach 150 cm/s, exceeding the measured top speed and driving effort toward the controller ceiling of $\pm420$ units. Effort can leave saturation during ramping or deceleration. A nominal 90-degree turn terminates at an absolute differential of 1440 encoder counts. This empirical threshold is distinct from the geometric conversion used in (\ref{eq:heading}). After each turn, the prescribed departure and stopping sequence is included in the recording. M0 is scored using the same footprint-clearance criterion as the turns.

\begin{table*}[t]
\centering
\caption{Maneuver parameters and principal-campaign commanded front-bumper distance grids. All turns have a nominal 90-degree encoder target.}
\label{tab:maneuvers}
\small
\begin{tabular}{@{}llrrrrrl@{}}
\toprule
 & Maneuver & $\alpha$ & $K_p,K_d$ & Ramp (s) & Decel. fraction & Exit/top & $d$ grid (cm) \\
\midrule
M0 & Hard brake & --- & $100,0$ & 0 & --- & 0 & 2, 4, 5, 6, 8, 11 \\
M1 & Driven-inner arc & 0.3 & $30,2$ & 0.25 & 0.30 & 0.6 & 11, 14, 16, 18, 21, 23 \\
M2 & Brake-assisted pivot & 0 & $100,0$ & 0 & 0.17 & 0.4 & 7, 10, 12, 14, 17, 21 \\
M3 & Reverse spin & $-1$ & $100,0$ & 0 & 0.17 & 0 & 4, 6, 8, 10, 12, 16 \\
M4 & Managed reverse spin & $-1$ & $100,0$ & 0 & 0.70 & 0.3 & 4, 6, 7, 9, 11, 14 \\
\bottomrule
\end{tabular}
\begin{minipage}{0.97\textwidth}\footnotesize
Deceleration fraction is the final fraction of the encoder-defined turn over which the target magnitude is reduced. $K_i=0$ during each maneuver; gains are implementation units. Exit/top is the exit-speed target divided by the measured top-speed reference. The nominal inner/outer ratio $\alpha$ describes the ideal motion used in (\ref{eq:index}), not a measured ratio throughout the transient.
\end{minipage}
\end{table*}

\subsection{Camera measurement and calibration}
A separate computer records a 1920$\times$1080 overhead video at a nominal 60 frames/s. A $4\times4$ ArUco marker on the robot provides pose, and four surveyed floor markers define a per-clip homography \cite{garridojurado2014}. Intrinsic calibration yields 0.56 pixel RMS reprojection error; the mean per-clip floor residual is approximately 0.64 cm. The tracked marker is 8 cm above the floor. With camera height $H$, nadir $N$, and floor-homography output $M$, height-induced parallax is corrected using
\begin{equation}
 P=N+\frac{H-h}{H}(M-N),\qquad h=8\text{ cm}.
 \label{eq:parallax}
\end{equation}
Camera height and nadir are obtained from the surveyed floor markers using perspective pose estimation; the saved calibration gives $H=76.8$ cm. The homothety in (\ref{eq:parallax}) changes marker position while preserving planar orientation.

Static validation at two surveyed marks gives mean position bias of 0.30 cm and mean frame-to-frame repeatability of 0.014 cm. These are position-validation statistics, not bounds on dynamic heading accuracy or on each measured clearance. Fiducial-pose error depends on imaging geometry and marker visibility \cite{kalaitzakis2021}. The footprint offsets relative to the tracked point are $a_f=9.0$ cm, $a_r=7.5$ cm, and $b=7.5$ cm. These define the rectangular approximation used for scoring.

\subsection{Trial design and quality control}
The principal campaign crosses five maneuvers, three commanded speeds (45, 55, and 65 cm/s), six maneuver-specific distances, and six repetitions, giving 540 trials. Three randomized blocks each contain 180 trials. For each turning maneuver--speed--distance combination, three repetitions turn left and three turn right. Trials were collected over two days; the day boundary occurred within a block. Every condition was represented on both days, avoiding the pilot's aliasing of maneuver and collection day.

The robot is initially stationary within a camera placement gate of 8 cm and 8 degrees. A virtual plane is defined 50 cm from the measured start reference point, perpendicular to the measured initial heading. The firmware triggers the maneuver after the prescribed approach distance. The camera supplies the outcome after the trial and is not used for feedback during avoidance. Each trial records camera pose, timestamped wheel speeds, gyro rate, maneuver-phase markers, and battery minimum voltage. Batteries are replaced during collection to limit voltage-related degradation; remaining day effects are considered in the regression.

A recorded heading-unwrapping flag excludes one M4 trial at the highest speed, leaving 539 trials: 36 per maneuver--speed condition except that condition, which has 35. All retained trials use the per-clip floor map. Six success classifications differ between the saved and per-clip homographies; these boundary-sensitive trials remain in the principal dataset. No trial is excluded merely because avoidance fails. 

The synchronization LED was obscured by the marker mount. Cross-correlation of motion signatures provides alignment in 456 of the 540 collected records, with a motion-onset fallback in 84. Eighty-two records have correlation below 0.6. Consequently, event-aligned transient plots require individual checks. The principal clearance, net-heading, and path-discrepancy statistics do not use the estimated lag. Whole-record integration is recomputed from the raw timestamped traces rather than from the firmware's turn-end summary.

\subsection{Secondary displacement-discrepancy metric}
For comparison with the original experiment design, define
\begin{equation}
 \begin{split}
 S_{\ell}&=\left|\sum_k\|p_{k+1}-p_k\|-\int\left|\frac{u_R+u_L}{2}\right|dt\right|,\\
 S_r&=\frac{L}{2}\bigl||\Theta_e|-|\Theta_c|\bigr|,\qquad
 S=\sqrt{S_{\ell}^2+S_r^2}.
 \end{split}
 \label{eq:composite}
\end{equation}
Angles in $S_r$ are in radians. Only camera steps strictly between 0.05 and 8.0 cm are included in the path sum; the lower cutoff suppresses stationary jitter and the upper cutoff rejects large pose jumps. The recorded $S$ values use these whole-run quantities before rounding.

The composite $S$ is a displacement-discrepancy score in centimeters. It is not a direct measurement of cumulative tire sliding: cancellation can occur, body lateral motion is not resolved into contact-patch slip, and calibration error contributes to the path difference. Both its definition and the distance grid affect comparisons. Heading error is therefore the primary sensing measure, and $S$ is a secondary empirical trade-off axis.

\subsection{Statistical analysis}\label{sec:stats}
For maneuver $m$ and commanded speed $s$, success is modeled as
\begin{equation}
 \begin{aligned}
 \logit P(Y=1\mid d,m,s)&=a_{ms}+b_{ms}d,\\
 \widehat{\dN}_{ms}&=\frac{\log 9-\hat a_{ms}}{\hat b_{ms}}.
 \end{aligned}
 \label{eq:d90}
\end{equation}
Firth penalization gives finite estimates in separated conditions \cite{firth1993,heinze2002}. Each fit uses 35 or 36 binary outcomes, and every fitted slope is positive. The 95\% intervals for $\dN$ are obtained by profiling the full Firth-penalized likelihood over the constraint $a+b\dN=\log9$, using a one-degree-of-freedom likelihood-ratio cutoff. These are approximate model-based intervals, not a guarantee of 90\% reliability for a future deployment. Extrapolation beyond the tested grid is identified when it occurs.

Means and lowest-to-highest-speed differences in $E_e$ use percentile intervals from 10,000 bootstrap resamples within each maneuver--speed--distance cell, preserving the number of trials at each distance. To examine adjusted association in the 431 turning trials, a separate ordinary logistic model is fitted:
\begin{equation}
 \begin{split}
 \logit P(Y_i=1)={}&\beta_0+\beta_d d_i+\beta_e E_{e,i}\\
 &+\gamma_{m_i}+\eta_{s_i}+\delta_{D_i}+\zeta_{T_i}.
 \end{split}
 \label{eq:assoc}
\end{equation}
Maneuver, commanded speed, collection day $D_i$, and turn direction $T_i$ enter as categorical effects. $E_e$ is in degrees and $d$ in centimeters. This model converges without separation. A stratified trial bootstrap uses 2000 resamples with the same maneuver--speed--distance allocation. Sensitivity models replace commanded-speed categories with the camera-derived approach-speed statistic, or allow maneuver-specific distance slopes. Reported $p$ values are two-sided Wald values and are exploratory, without multiplicity adjustment.

Clearance variability is described after subtracting a fitted linear distance trend separately within each maneuver--speed condition. A point-estimate Pareto comparison minimizes $\widehat{\dN}$ and mean $S$. A maneuver is nondominated when no other tested maneuver has no larger values on both axes and a smaller value on at least one. Frontier membership itself is not treated as a statistically established ordering.

\section{Experimental Results}\label{sec:results}
\subsection{Campaign context and speed control}
The 300-trial pilot suggested an advantage of M4 over M3, but maneuver identity was confounded with collection day and an approximately 4.3 cm/s approach-speed difference. That comparison cannot establish a maneuver effect. A subsequent 210-trial experiment held the commanded speed at 45 cm/s while varying deceleration shaping with other parameters fixed. It did not detect a change in reaction-distance requirement or the original slip statistic. This null result motivated varying approach speed; it does not establish equivalence of the deceleration settings or prove that slip is absent at low speed.

The 539-trial principal dataset supplies all comparisons below. The mean camera-derived approach-speed statistics are 47.7, 58.2, and 68.7 cm/s for commands of 45, 55, and 65 cm/s, respectively. The per-trial statistic is the 90th percentile of smoothed camera speed in the selected straight-approach region, not an average over the entire approach. Differences in its mean across maneuvers are less than 0.5 cm/s within a commanded-speed level. Thus the commanded levels are well separated, while the nominal command should not be identified with exact ground speed.

Pooled success fractions are approximately 82\%, 68\%, and 48\% across the three speeds. Mean minimum clearance decreases from 3.52 to 2.39 to 0.51 cm. Because the same distance grid is used at all speeds within a maneuver, these summarize the speed trend under the experimental allocation. Raw success fractions should not be used to rank maneuvers whose distance grids differ.

\subsection{Encoder discrepancy increases for the reverse spins}
Figure~\ref{fig:heading} shows the whole-record heading discrepancies. For M3, mean $E_e$ increases from 3.80 to 4.71 to 8.90 degrees. For M4 it increases from 4.60 to 7.02 to 9.33 degrees. In contrast, M1 remains near 6.2 degrees and M2 remains between 4.4 and 5.3 degrees. Table~\ref{tab:contrast} reports the low-to-high-speed changes with uncertainty.

\begin{figure*}[t]
 \centering
 \includegraphics[width=\linewidth]{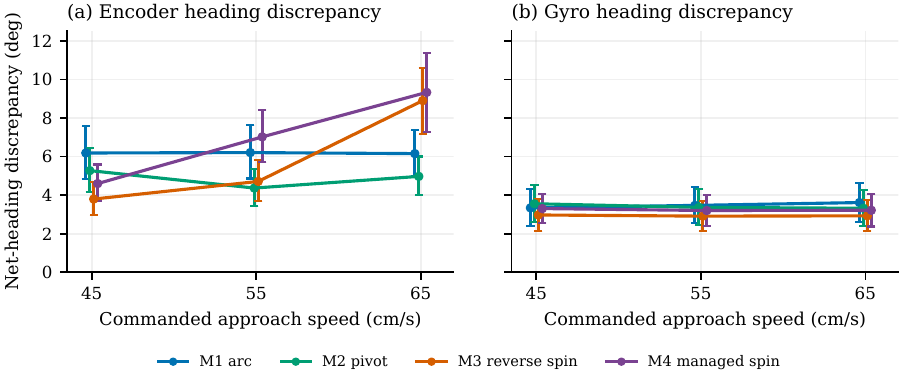}
 \caption{Whole-record net-heading discrepancies for the four turning maneuvers. Points are condition means; bars are 95\% bootstrap intervals from resampling trials within distance cells. Each condition contains 36 trials except M4 at 65 cm/s ($n=35$). Horizontal offsets separate overlapping markers. Encoder discrepancy grows with speed for M3 and M4, whereas the integrated gyro discrepancies remain approximately constant.}
 \label{fig:heading}
\end{figure*}

\begin{table}[t]
\centering
\caption{Change in mean $E_e$ from 45 to 65 cm/s command, in degrees. Intervals are exploratory 95\% bootstrap intervals.}
\label{tab:contrast}
\begin{tabular}{@{}lrr@{}}
\toprule
Maneuver & Mean change & Interval \\
\midrule
\contrastrows
\bottomrule
\end{tabular}
\end{table}
The two profiles with $\Nr=1$ exhibit a clear increase in net encoder discrepancy over the tested speed range. Because $\Nr$ is fixed within maneuver, this comparison does not separate it from gains, transients, and exit behavior.

\subsection{Gyro discrepancy and measurement interval}
Using the same complete-motion convention, mean $E_g$ is 2.91--2.97 degrees for M3 and 3.20--3.30 degrees for M4. The arc and pivot are also approximately stable, at 3.3--3.6 degrees. These results do not show a speed-dependent loss of gyro heading accuracy comparable to the encoder trend.

The distinction between complete records and firmware summaries is consequential. For M3, comparing the stored gyro summary with final camera heading produces mean discrepancies of 15.4--18.2 degrees. In Fig.~\ref{fig:window}, the summary of 70.2 degrees matches the logged gyro state 26 ms before turn completion. Rate integration gives 82.8 degrees during the commanded turn and a further 6.3 degrees afterward; the whole-record gyro and camera rotations are 89.3 and 88.7 degrees, respectively. Thus both the earlier snapshot and recovery contribute to the original mismatch. These observations establish the need for consistent measurement intervals; they do not validate a gyro-feedback controller, which was not tested.

\subsection{Encoder discrepancy and avoidance success}
For (\ref{eq:assoc}), $\hat\beta_e=-0.222$ per degree, with $p=7.8\times10^{-4}$. The corresponding odds ratio for success is 0.80 per additional degree of net encoder discrepancy. The 95\% stratified-bootstrap interval is $[-0.431,-0.105]$ for the coefficient, or $[0.65,0.90]$ for the odds ratio; all 2000 resampled fits converge.

The association persists when replacing commanded-speed categories with the camera-derived approach statistic ($\hat\beta_e=-0.231$, $p=6.0\times10^{-4}$), and when allowing maneuver-specific distance slopes ($\hat\beta_e=-0.221$, $p=7.8\times10^{-4}$). Thus the result is not explained solely by the common-distance-slope specification or by using commanded speed.

This is an adjusted association with an outcome measured after execution. $E_e$ is neither a randomized intervention nor information available to the encoder-only controller during the maneuver. The camera contributes to both heading and clearance estimates, so correlated measurement errors may also contribute. The regression does not establish that slip causes failure, and the odds ratio is not an absolute change in success probability.

\begin{figure*}[t]
 \centering
 \includegraphics[width=\linewidth]{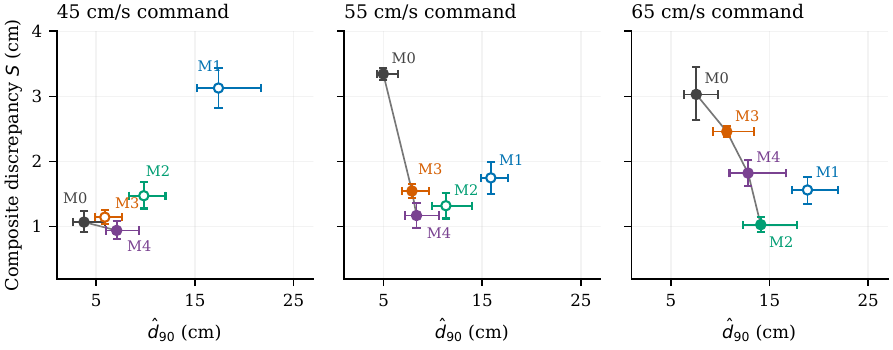}
 \caption{Speed-specific reaction-distance/displacement-discrepancy comparisons. Lower values are preferred. Filled markers indicate point-estimate nondominated maneuvers; hollow markers show the remaining maneuvers. Horizontal bars are 95\% profile intervals for $\dN$ and vertical bars are 95\% bootstrap intervals for mean $S$. Gray lines join the point-estimate frontier and do not represent tested intermediate maneuvers. Overlapping uncertainty and the sensitivity of $S$ preclude a definitive frontier-membership claim.}
 \label{fig:pareto}
\end{figure*}

\subsection{Clearance variability}
After removing the distance trend within each maneuver--speed condition, the pooled residual standard deviations are 0.88, 0.90, and 1.42 cm. For M4, the residual standard deviation increases from 0.57 to 2.02 cm between the lowest and highest speeds. The higher-speed measurements therefore show greater residual scatter, especially for the managed reverse spin. With only three speed levels and no independent surface replication, this observation does not identify a sharp transition from deterministic to stochastic behavior.

\subsection{Reaction distance and empirical trade-offs}
Table~\ref{tab:d90} reports the fitted reaction distance for 90\% success. All point estimates lie within the corresponding tested grid. The upper interval limit for M4 at 65 cm/s extends beyond its largest tested distance of 14 cm, indicating uncertainty beyond the sampled range. The fitted success curves and observed proportions are provided in Appendix~\ref{app:curves}.

\begin{table*}[t]
\centering
\caption{Estimated commanded front-bumper reaction distance $\widehat{\dN}$ (cm), with 95\% profile penalized-likelihood intervals.}
\label{tab:d90}
\begin{tabular}{@{}lccc@{}}
\toprule
Maneuver & 45 cm/s command & 55 cm/s command & 65 cm/s command \\
\midrule
\distancerows
\bottomrule
\end{tabular}
\end{table*}

M0 has the lowest point-estimate distance requirement at each speed, increasing from 3.8 to 7.6 cm. Among turns, M3 increases from 5.9 to 10.7 cm, M4 from 7.1 to 12.8 cm, and M2 from 9.8 to 14.1 cm. M1 requires the most room and is nonmonotonic (17.4, 15.9, and 18.8 cm). Therefore, steering does not reduce the estimated reaction distance below hard braking in this tested setting.

Figure~\ref{fig:pareto} displays all maneuvers on the $\widehat{\dN}$--mean-$S$ plane, with uncertainty on both axes. The point-estimate nondominated set is M0/M4 at 45 cm/s, M0/M3/M4 at 55 cm/s, and M0/M3/M4/M2 at 65 cm/s. M2 joins the highest-speed frontier with the lowest mean composite discrepancy, at the cost of a larger distance requirement. M1 is dominated in this two-metric comparison. M3 and M4 trade shorter reaction distance against lower composite discrepancy at the two higher speeds; the controlled data do not reproduce a universal M4 advantage.

The composite metric is sensitive to processing and scale. With lower camera-step thresholds of 0, 0.02, 0.05, and 0.10 cm, the pooled mean $S$ values at the three speeds are, respectively, $(2.04,2.63,2.95)$, $(1.58,1.89,2.06)$, $(1.55,1.83,1.98)$, and $(1.57,1.79,1.88)$ cm. The pooled speed ordering persists, but absolute magnitudes change. The camera-derived approach statistic also exceeds the corresponding late-approach wheel-speed statistic by approximately 6\%; their different temporal definitions and possible scale error prevent interpreting this offset directly as wheel slip. These checks support treating $S$ as an empirical discrepancy score, not an identified physical sliding distance.

\section{Discussion}\label{sec:discussion}
\subsection{Implications for maneuver selection and sensing}
The results distinguish three considerations in emergency motion. First, footprint geometry imposes a clearance cost even for an ideal rotation about a fixed reference point. Second, the measured reaction-distance requirement increases with speed for braking and the three tighter turning maneuvers. Third, the reverse-spin profiles develop larger encoder discrepancies, while the whole-record gyro discrepancies remain comparatively stable.

These observations favor characterizing each candidate maneuver at the relevant approach speed and evaluating it against the available distance and task objective. Hard braking requires the least room in this experiment, but it does not perform a commanded change of heading. If redirection or continued motion is required, a distance-only comparison is insufficient. Conversely, comparing a braking and a turning maneuver on $S$ alone does not establish which best serves a navigation task. The empirical frontiers describe the chosen two metrics and experimental allocation, rather than a complete operational utility or a guaranteed safety envelope.

The low-cost gyro provides useful additional information in these records. This motivates evaluating gyro-assisted termination and recovery control. Gyrodometry \cite{borenstein1996gyro} and encoder--IMU slip estimation \cite{nourizadeh2023} provide relevant precedents for using onboard sensing; the present endpoint comparison does not establish which estimator or feedback law would perform best during a rapid transient. A camera remains valuable as an independent experimental reference. A controlled feedback comparison would be needed to determine whether using the gyro reduces final error or improves avoidance success.

\subsection{Limitations}
The study uses one robot, one tile surface, one camera arrangement, and two collection days. The experiment does not measure friction, contact forces, load transfer, or motor torque, and the nominal index does not account for these quantities. Maneuver identity and index are not independently varied.

Camera validation is static and positional, with two surveyed marks. Dynamic heading error, marker-mount rigidity, footprint approximation, and between-frame plane crossings remain limitations. Trial resampling does not propagate calibration uncertainty or temporal dependence within a collection session. The fixed day term adjusts an average day difference but does not supply independent session replication. Classifications near zero clearance are sensitive to the floor map, as shown by the six map-dependent outcomes.

The reaction-distance coordinate is commanded through encoder-based approach travel. It is not an independently measured detection-to-obstacle distance, and the experiment excludes detection latency and moving obstacles. Net-heading discrepancies can conceal cancellation of transient errors. The composite displacement metric is sensitive to camera path processing and wheel calibration. Finally, the modest replication at each distance limits the precision of fitted 90\% thresholds, and exploratory associations cannot establish a causal slip mechanism.

\section{Conclusion}\label{sec:conclusion}
Five emergency maneuvers were evaluated on a differential-drive robot using independent camera pose measurements. A geometric calculation quantified the footprint-clearance cost of an ideal rotation, and a normalized wheel-motion index described the tested turning profiles without assuming an identified friction limit. In 539 retained controlled trials, mean whole-record encoder discrepancy increased by approximately 4.7--5.1 degrees for the two reverse-spin maneuvers across the tested speed range, while the arc and pivot changed little. Larger encoder discrepancy was associated with lower success after covariate adjustment, and clearance scatter increased at the highest speed.

Whole-record gyro discrepancies remained near 3 degrees for the reverse spins. Comparing a stored turn-summary snapshot with a final post-recovery pose substantially overstated gyro error, demonstrating why sample timing and a common motion interval matter. The reaction-distance estimates and empirical frontiers support speed-specific characterization of maneuver trade-offs. Future experiments should vary surface and robot parameters, validate dynamic heading accuracy, measure actual trigger clearance, and compare encoder-only control with gyro-assisted turn and recovery feedback.

\clearpage
\onecolumn
\appendices
\section{Observed Success Fractions and Fitted Curves}\label{app:curves}
Figure~\ref{fig:curves} shows the binary-outcome data aggregated by distance and the fitted curves used in Table~\ref{tab:d90}. Each observed proportion is based on six trials, except the excluded M4 cell, which contains five. Firth penalization permits estimation when observed outcomes separate; it does not add observations between distance levels. The analysis package contains trial identifiers, complete-record sensor integrals, success counts, deterministic analysis scripts, and hashes of the raw inputs used for verification.

\begin{figure}[ht]
 \centering
 \includegraphics[width=\textwidth]{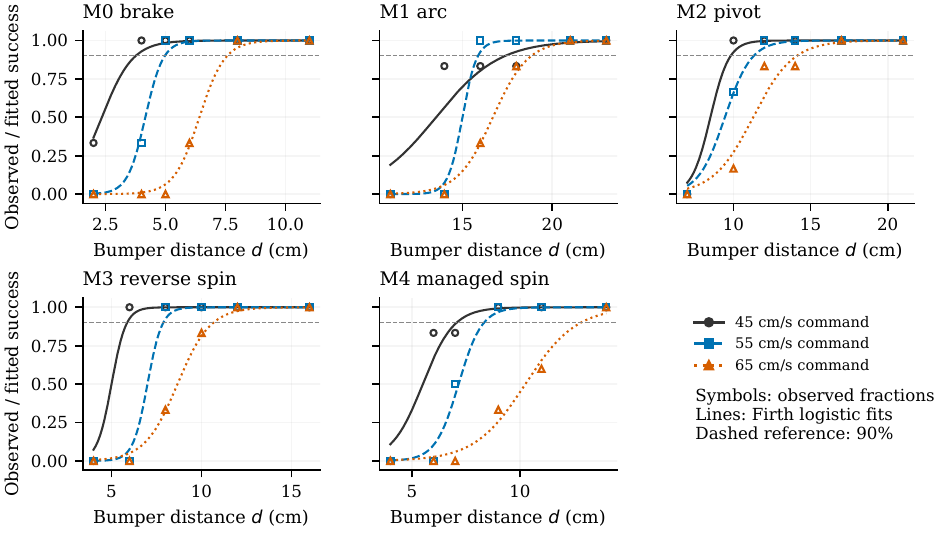}
 \caption{Observed success fractions and Firth logistic fits for every maneuver and commanded speed. Curves are drawn only over each maneuver's tested distance range. Multiple symbols can coincide at zero or unit success. These data support the distance estimates, while the profile intervals in Table~\ref{tab:d90} express model-based uncertainty.}
 \label{fig:curves}
\end{figure}
\end{document}

%% file: contrast_rows.tex
\newcommand{\contrastrows}{%
M1 & -0.03 & [-1.85, +1.80] \\
M2 & -0.30 & [-1.82, +1.22] \\
M3 & +5.11 & [+3.15, +7.07] \\
M4 & +4.73 & [+2.47, +6.97] \\}

%% file: d90_rows.tex
\newcommand{\distancerows}{%
M0 & 3.8 [2.7, 5.9] & 5.0 [4.3, 6.5] & 7.6 [6.4, 9.7] \\
M1 & 17.4 [15.2, 21.7] & 15.9 [14.9, 17.6] & 18.8 [17.3, 22.0] \\
M2 & 9.8 [8.4, 12.0] & 11.3 [9.9, 14.0] & 14.1 [12.3, 17.8] \\
M3 & 5.9 [4.9, 7.6] & 7.9 [6.9, 9.6] & 10.7 [9.3, 13.4] \\
M4 & 7.1 [6.0, 9.4] & 8.3 [7.2, 10.6] & 12.8 [10.9, 16.6] \\}